\def\year{2020}\relax
\documentclass[letterpaper]{article}

\usepackage{aaai20}
\usepackage{times}
\usepackage{helvet}
\usepackage{courier}
\usepackage[hyphens]{url}
\usepackage{graphicx}
\usepackage{booktabs}
\usepackage{bbding} 
\usepackage{amsmath,amsthm}
\usepackage{xcolor}
\usepackage{subcaption}
\usepackage{amsfonts}
\usepackage[ruled,vlined,shortend, linesnumbered]{algorithm2e} 
\usepackage[]{algorithmic} 
\usepackage{enumitem}
\newcommand{\method}{\textsc{Midas}}
\newcommand{\sedanspot}{\textsc{SedanSpot}}

\newtheorem{definition}{Definition}

\newtheorem{theorem}{Theorem}
\newcommand{\codeurl}{https://github.com/bhatiasiddharth/MIDAS}

\title{\method
: Microcluster-Based Detector of Anomalies in Edge Streams}
\author{Siddharth Bhatia,\textsuperscript{1}
Bryan Hooi,\textsuperscript{1}
Minji Yoon,\textsuperscript{2}
Kijung Shin,\textsuperscript{3}
Christos Faloutsos\textsuperscript{2}\\
\textsuperscript{1}{National University of Singapore},
\textsuperscript{2}{Carnegie Mellon University},
\textsuperscript{3}{KAIST}\\
siddharth@comp.nus.edu.sg, bhooi@comp.nus.edu.sg, minjiy@cs.cmu.edu, kijungs@kaist.ac.kr, christos@cs.cmu.edu
}

\begin{document}

\maketitle
\begin{abstract}
Given a stream of graph edges from a dynamic graph, how can we assign anomaly scores to edges in an online manner, for the purpose of detecting unusual behavior, using constant time and memory? Existing approaches aim to detect \emph{individually surprising} edges. In this work, we propose \method, which focuses on detecting \emph{microcluster anomalies}, or suddenly arriving groups of suspiciously similar edges, such as lockstep behavior, including denial of service attacks in network traffic data. \method\ has the following properties: (a) it detects microcluster anomalies while providing theoretical guarantees about its false positive probability; (b) it is online, thus processing each edge in constant time and constant memory, and also processes the data $162-644$ times faster than state-of-the-art approaches; (c) it provides $42\%$-$48\%$ higher accuracy (in terms of AUC) than state-of-the-art approaches.
\end{abstract}
\section{Introduction}
Anomaly detection in graphs is a critical problem for finding suspicious behavior in innumerable systems, such as intrusion detection, fake ratings, and financial fraud. This has been a well-researched problem with majority of the proposed approaches \cite{akoglu2010oddball,chakrabarti2004autopart,hooi2017graph,jiang2016catching,kleinberg1999authoritative,shin2018patterns} focusing on static graphs. However, many real-world graphs are dynamic in nature, and methods based on static connections may miss temporal characteristics of the graphs and anomalies.

Among the methods focusing on dynamic graphs, most of them have edges aggregated into graph snapshots \cite{eswaran2018spotlight,sun2006beyond,sun2007graphscope,koutra2013deltacon,Sricharan,Gupta}. However, to minimize the effect of malicious activities and start recovery as soon as possible, we need to detect anomalies in real-time or near real-time i.e. to identify whether an incoming edge is anomalous or not, as soon as we receive it. In addition, since the number of vertices can increase as we process the stream of edges, we need an algorithm which uses constant memory in graph size.

Moreover, fraudulent or anomalous events in many applications occur in microclusters or suddenly arriving groups of suspiciously similar edges e.g. denial of service attacks in network traffic data and lockstep behavior. However, existing methods which process edge streams in an online manner, including \cite{eswaran2018sedanspot,ranshous2016scalable}, aim to detect individually surprising edges, not microclusters, and can thus miss large amounts of suspicious activity. 

In this work, we propose \method, which detects \emph{microcluster anomalies}, or suddenly arriving groups of suspiciously similar edges, in edge streams, using constant time and memory. In addition, by using a principled hypothesis testing framework, \method\ provides theoretical bounds on the false positive probability, which these methods do not provide. 

Our main contributions are as follows:
\begin{enumerate}
\item Streaming Microcluster Detection: We propose a novel streaming approach for detecting microcluster anomalies, requiring constant time and memory. 
\item Theoretical Guarantees: In Theorem \ref{thm:bound}, we show guarantees on the false positive probability of \method.
\item Effectiveness: Our experimental results show that \method\ outperforms baseline approaches by $42\%$-$48\%$ accuracy (in terms of AUC), and processes the data $162-644$ times faster than baseline approaches. 
\end{enumerate}
Reproducibility: Our code and datasets are publicly available at \codeurl.

\section{Related Work}
In this section, we review previous approaches to detect anomalous signs on static and dynamic graphs. 
See \cite{akoglu2015graph} for an extensive survey on graph-based anomaly detection.\\
\noindent{\bf Anomaly detection in static graphs} can be classified by which anomalous entities (nodes, edges, subgraph, etc.) are spotted. 
\begin{itemize}
\item Anomalous node detection:
\cite{akoglu2010oddball} extracts egonet-based features and finds empirical patterns with respect to the features. 
Then, it identifies nodes whose egonets deviate from the patterns, including the count of triangles, total weight, and principal eigenvalues.
\cite{jiang2016catching} computes node features, including degree and authoritativeness~\cite{kleinberg1999authoritative}, then spots nodes whose neighbors are notably close in the feature space.
\item Anomalous subgraph detection:
\cite{hooi2017graph} and \cite{shin2018patterns} measure the anomalousness of nodes and edges, detecting a dense subgraph consisting of many anomalous nodes and edges. 
\item Anomalous edge detection:
\cite{chakrabarti2004autopart} encodes an input graph based on similar connectivity among nodes, then spots edges whose removal reduces the total encoding cost significantly.
\cite{tong2011non} factorize the adjacency matrix and flag edges with high reconstruction error as outliers. 
\end{itemize}
\noindent{\bf Anomaly detection in graph streams} use as input a series of graph snapshots over time. We categorize them similarly according to the type of anomaly detected:
\begin{itemize}{
\item Anomalous node detection:
\cite{sun2006beyond} approximates the adjacency matrix of the current snapshot based on incremental matrix factorization, then spots nodes corresponding to rows with high reconstruction error.
\item Anomalous subgraph detection:
Given a graph with timestamps on edges, \cite{beutel2013copycatch} spots near-bipartite cores where each node is connected to others in the same core densly within a short time. 
\cite{jiang2016catching} detects groups of nodes who form dense subgraphs in a temporally synchronized manner.
\item Anomalous event detection:
\cite{eswaran2018spotlight} detects sudden appearance of many unexpected edges, and \cite{yoon2019fast} spots sudden changes in 1st and 2nd derivatives of PageRank.
}\end{itemize}
\noindent{\bf Anomaly detection in edge streams} use as input a stream of edges over time. Categorizing them according to the type of anomaly detected:
\begin{itemize}
\item Anomalous node detection:
Given an edge stream, \cite{yu2013anomalous} detects nodes whose egonets suddenly and significantly change.
\item Anomalous subgraph detection:
Given an edge stream, \cite{shin2017densealert} identifies dense subtensors created within a short time.
\item Anomalous edge detection:
\cite{ranshous2016scalable} focuses on sparsely-connected parts of a graph, while \cite{eswaran2018sedanspot} identifies edge anomalies based on edge occurrence, preferential attachment, and mutual neighbors.
\end{itemize}
\noindent Only the 2 methods in the last category are applicable to our task, as they operate on edge streams and output a score per edge. However, as shown in Table \ref{tab:comparison}, neither method aims to detect microclusters, or provides guarantees on false positive probability. 

\begin{table}[!ht]
\centering
\caption{Comparison of relevant edge stream anomaly detection approaches.}
\label{tab:comparison}
\begin{tabular}{@{}rcc|c@{}}
\toprule
 & \rotatebox{90}{\sedanspot~\shortcite{eswaran2018sedanspot}} 
 & \rotatebox{90}{RHSS~\shortcite{ranshous2016scalable}}
 & {\bf \rotatebox{90}{\method}} \\ \midrule
\textbf{Microcluster Detection} & & & \CheckmarkBold \\
\textbf{Guarantee on False Positive Probability} & & & \CheckmarkBold \\
\textbf{Constant Memory} & \Checkmark & \Checkmark & \CheckmarkBold \\
\textbf{Constant Update Time} & \Checkmark & \Checkmark & \CheckmarkBold \\
\bottomrule
\end{tabular}
\end{table}

\section{Problem}

Let $\mathcal{E} = \{e_1, e_2, \cdots\}$ be a stream of edges from a time-evolving graph $\mathcal{G}$. Each arriving edge is a tuple $e_i = (u_i, v_i, t_i)$ consisting of a source node $u_i \in \mathcal{V}$, a destination node $v_i \in \mathcal{V}$, and a time of occurrence $t_i$, which is the time at which the edge was added to the graph. For example, in a network traffic stream, an edge $e_i$ could represent a connection made from a source IP address $u_i$ to a destination IP address $v_i$ at time $t_i$. We do not assume that the set of vertices $\mathcal{V}$ is known a priori: for example, new IP addresses or user IDs may be created over the course of the stream.

We model $\mathcal{G}$ as a directed graph. Undirected graphs can simply be handled by treating an incoming undirected $e_i = (u_i, v_i, t_i)$ as two simultaneous directed edges, one in either direction.

We also allow $\mathcal{G}$ to be a multigraph: edges can be created multiple times between the same pair of nodes. Edges are allowed to arrive simultaneously: i.e. $t_{i+1} \ge t_i$, since in many applications $t_i$ are given in the form of discrete time ticks. 

The desired properties of our algorithm are as follows:

\begin{itemize}
\item {\bf Microcluster Detection:} It should detect suddenly appearing bursts of activity which share many repeated nodes or edges, which we refer to as microclusters.
\item {\bf Guarantees on False Positive Probability:} Given any user-specified probability level $\epsilon$ (e.g. $1\%$), the algorithm should be adjustable so as to provide false positive probability of at most $\epsilon$ (e.g. by adjusting a threshold that depends on $\epsilon$). Moreover, while guarantees on the false positive probability rely on assumptions about the data distribution, we aim to make our assumptions as weak as possible.
\item {\bf Constant Memory and Update Time:} For scalability in the streaming setting, the algorithm should run in constant memory and constant update time per newly arriving edge. Thus, its memory usage and update time should not grow with the length of the stream, or the number of nodes in the graph. 
\end{itemize}

\section{Proposed Algorithm}
\subsection{Overview}

Next, we describe our \method\ and \method-R approaches. The following provides an overview:

\begin{enumerate}
\item {\bf Streaming Hypothesis Testing Approach:} We describe our \method\ algorithm, which uses streaming data structures within a hypothesis testing-based framework, allowing us to obtain guarantees on false positive probability.
\item {\bf Detection and Guarantees:} We describe our decision procedure for determining whether a point is anomalous, and our guarantees on false positive probability.
\item {\bf Incorporating Relations:} We extend our approach to the \method-R algorithm, which incorporates relationships between edges temporally and spatially\footnote{We use `spatially' in a graph sense, i.e. connecting nearby nodes, not to refer to any other continuous spatial dimension.}.
\end{enumerate}

\subsection{\method: Streaming Hypothesis Testing Approach}

\begin{figure}[!htb]
        \center{\includegraphics[width=\columnwidth]
        {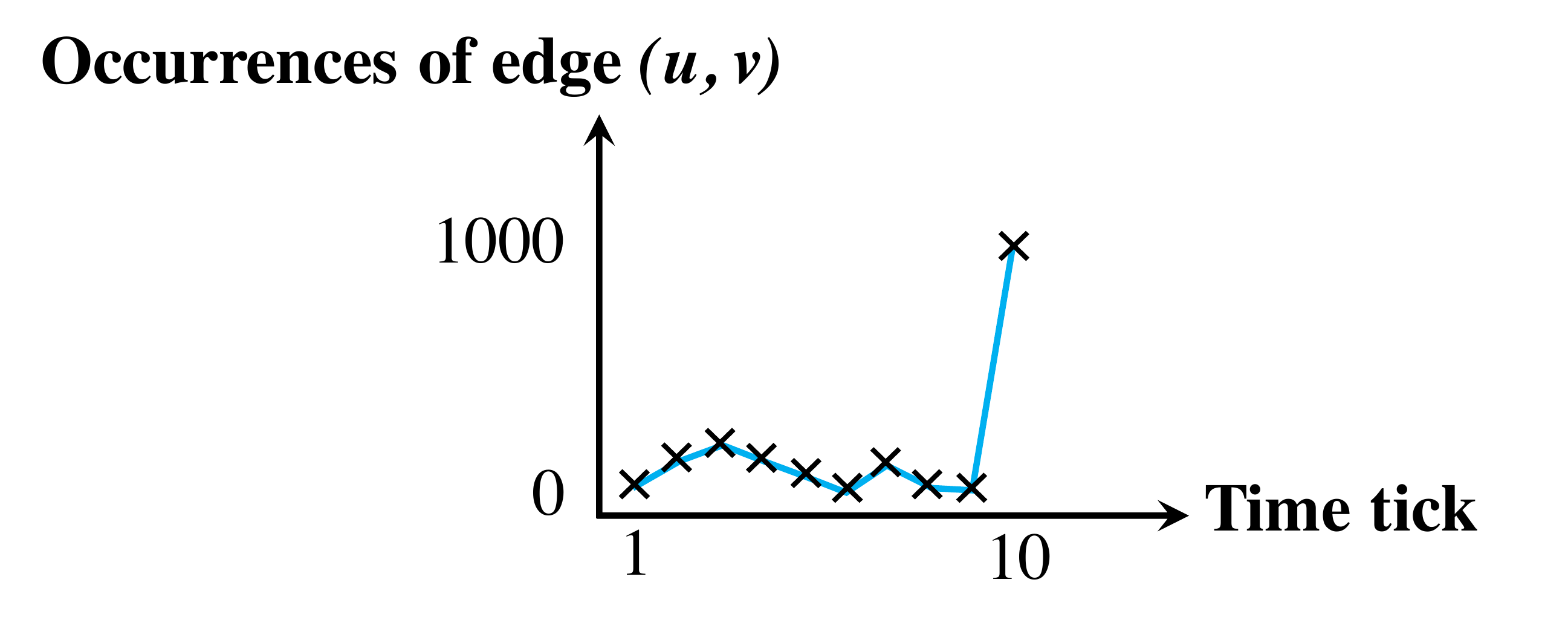}}
        \caption{\label{fig:intro} Time series of a single source-destination pair $(u,v)$, with a large burst of activity at time tick $10$.}
\end{figure}

Consider the example in Figure \ref{fig:intro} of a single source-destination pair $(u,v)$, which shows a large burst of activity at time $10$. This burst is the simplest example of a microcluster, as it consists of a large group of edges which are very similar to one another (in fact identical), both {\bf spatially} (i.e. in terms of the nodes they connect) and {\bf temporally}. 

\subsubsection{Streaming Data Structures}

In an offline setting, there are many time-series methods which could detect such bursts of activity. However, in an online setting, recall that we want memory usage to be bounded, so we cannot keep track of even a single such time series. Moreover, there are many such source-destination pairs, and the set of sources and destinations is not fixed a priori. 

To circumvent these problems, we maintain two types of Count-Min-Sketch (CMS)~\cite{cormode2005improved} data structures. Assume we are at a particular fixed time tick $t$ in the stream; we treat time as a discrete variable for simplicity. Let $s_{uv}$ be the total number of edges from $u$ to $v$ up to the current time. Then, we use a single CMS data structure to approximately maintain all such counts $s_{uv}$ (for all edges $uv$) in constant memory: at any time, we can query the data structure to obtain an approximate count $\hat{s_{uv}}$. 

Secondly, let $a_{uv}$ be the number of edges from $u$ to $v$ in the current time tick (but not including past time ticks). We keep track of $a_{uv}$ using a similar CMS data structure, the only difference being that we reset this CMS data structure every time we transition to the next time tick. Hence, this CMS data structure provides approximate counts $\hat{a_{uv}}$ for the number of edges from $u$ to $v$ in the current time tick $t$. 

\subsubsection{Hypothesis Testing Framework}

Given approximate counts $\hat{s_{uv}}$ and $\hat{a_{uv}}$, how can we detect microclusters? Moreover, how can we do this in a principled framework that allows for theoretical guarantees? 

Fix a particular source and destination pair of nodes, $(u,v)$, as in Figure \ref{fig:intro}. One approach would be to assume that the time series in Figure \ref{fig:intro} follows a particular generative model: for example, a Gaussian distribution. We could then find the mean and standard deviation of this Gaussian distribution. Then, at time $t$, we could compute the Gaussian likelihood of the number of edge occurrences in the current time tick, and declare an anomaly if this likelihood is below a specified threshold. 

However, this requires a restrictive Gaussian assumption, which can lead to excessive false positives or negatives if the data follows a very different distribution. Instead, we use a weaker assumption: that the mean level (i.e. the average rate at which edges appear) in the current time tick (e.g. $t=10$) is the same as the mean level before the current time tick $(t<10)$. Note that this avoids assuming any particular distribution for each time tick, and also avoids a strict assumption of stationarity over time.

Hence, we can divide the past edges into two classes: the current time tick $(t=10)$ and all past time ticks $(t<10)$. Recalling our previous notation, the number of events at $(t=10)$ is $a_{uv}$, while the number of edges in past time ticks $(t<10)$ is $s_{uv} - a_{uv}$. 

Under the chi-squared goodness-of-fit test, the chi-squared statistic is defined as the sum over categories of $\frac{(\text{observed} - \text{expected})^2}{\text{expected}}$. In this case, our categories are $t=10$ and $t<10$. Under our mean level assumption, since we have $s_{uv}$ total edges (for this source-destination pair), the expected number at $t=10$ is $\frac{s_{uv}}{t}$, and the expected number for $t<10$ is the remaining, i.e. $\frac{t-1}{t} s_{uv}$. Thus the chi-squared statistic is:

\begin{align*}
X^2 &= \frac{(\text{observed}_{(t=10)} - \text{expected}_{(t=10)})^2}{\text{expected}_{(t=10)}} \\
&+ \frac{(\text{observed}_{(t<10)} - \text{expected}_{(t<10)})^2}{\text{expected}_{(t<10)}}\\
&= \frac{(a_{uv} - \frac{s_{uv}}{t})^2}{\frac{s_{uv}}{t}} + \frac{((s_{uv} - a_{uv}) - \frac{t-1}{t} s_{uv})^2}{\frac{t-1}{t} s_{uv}}\\
&= \frac{(a_{uv} - \frac{s_{uv}}{t})^2}{\frac{s_{uv}}{t}} + \frac{(a_{uv} - \frac{s_{uv}}{t})^2}{\frac{t-1}{t} s_{uv}}\\
&= (a_{uv} - \frac{s_{uv}}{t})^2 \frac{t^2}{s_{uv}(t-1)}
\end{align*}
Note that both $a_{uv}$ and $s_{uv}$ can be estimated by our CMS data structures, obtaining approximations $\hat{a_{uv}}$ and $\hat{s_{uv}}$ respectively. This leads to our following anomaly score, using which we can evaluate a newly arriving edge with source-destination pair $(u,v)$:

\begin{definition}[Anomaly Score]
Given a newly arriving edge $(u,v,t)$, our anomaly score is computed as:
\begin{align}
\text{score}((u,v,t)) = (\hat{a_{uv}} - \frac{\hat{s_{uv}}}{t})^2 \frac{t^2}{\hat{s_{uv}}(t-1)}
\end{align}
\end{definition}

Algorithm \ref{alg:midas} summarizes our \method\ algorithm.

\begin{algorithm}
	\caption{\method:\ Streaming Anomaly Scoring \label{alg:midas}}
	\KwIn{Stream of graph edges over time}
	\KwOut{Anomaly scores per edge}
	{\bf $\triangleright$ Initialize CMS data structures:} \\
	Initialize CMS for total count $s_{uv}$ and current count $a_{uv}$ \\
	\While{new edge $e=(u,v,t)$ is received:}{
	{\bf $\triangleright$ Update Counts:} \\
	Update CMS data structures for the new edge $uv$\\
	{\bf $\triangleright$ Query Counts:} \\
	Retrieve updated counts $\hat{s_{uv}}$ and $\hat{a_{uv}}$\\
	{\bf $\triangleright$ Anomaly Score:}\\
	{\bf output} $\text{score}((u,v,t)) = (\hat{a_{uv}} - \frac{\hat{s_{uv}}}{t})^2 \frac{t^2}{\hat{s_{uv}}(t-1)}$\\
	}
\end{algorithm}

\subsection{Detection and Guarantees}
While Algorithm \ref{alg:midas} computes an anomaly score for each edge, it does not provide a binary decision for whether an edge is anomalous or not. We want a decision procedure that provides binary decisions and a guarantee on the false positive probability: i.e. given a user-defined threshold $\epsilon$, the probability of a false positive should be at most $\epsilon$. Intuitively, the key idea is to combine the approximation guarantees of CMS data structures with properties of a chi-squared random variable.

The key property of CMS data structures we use is that given any $\epsilon$ and $\nu$, for appropriately chosen CMS data structure sizes, with probability at least $1-\frac{\epsilon}{2}$, the estimates $\hat{a_{uv}}$ satisfy:
\begin{align}
\hat{a_{uv}} \le a_{uv} + \nu \cdot N_t
\end{align}
where $N_t$ is the total number of edges at time $t$. Since CMS data structures can only overestimate the true counts, we additionally have 
\begin{align}
s_{uv} \le \hat{s_{uv}}
\end{align}
Define an adjusted version of our earlier score:
\begin{align}
\tilde{a_{uv}} = \hat{a_{uv}} - \nu N_t
\end{align}
To obtain its probabilistic guarantee, our decision procedure computes $\tilde{a_{uv}}$, and uses it to compute an adjusted version of our earlier statistic: 
\begin{align}
\tilde{X^2} = (\tilde{a_{uv}} - \frac{\hat{s_{uv}}}{t})^2 \frac{t^2}{\hat{s_{uv}}(t-1)}
\end{align}
Then our main guarantee is as follows:
\begin{theorem}[False Positive Probability Bound] \label{thm:bound}
Let $\chi_{1-\epsilon/2}^2(1)$ be the $1-\epsilon/2$ quantile of a chi-squared random variable with 1 degree of freedom. Then:
\begin{align}
P(\tilde{X^2} > \chi_{1-\epsilon/2}^2(1)) < \epsilon
\end{align}
In other words, using $\tilde{X^2}$ as our test statistic and threshold $\chi_{1-\epsilon/2}^2(1)$ results in a false positive probability of at most $\epsilon$. 
\end{theorem}
\begin{proof}
Recall that 
\begin{align}
X^2 = (a_{uv} - \frac{s_{uv}}{t})^2 \frac{t^2}{s_{uv}(t-1)}
\end{align}
was defined so that it has a chi-squared distribution. Thus:
\begin{align} \label{eq:cond1}
P(X^2 \le \chi_{1-\epsilon/2}^2(1)) = 1-\epsilon/2
\end{align}
At the same time, by the CMS guarantees we have:
\begin{align} \label{eq:cond2}
P(\hat{a_{uv}} \le a_{uv} + \nu \cdot N_t) \ge 1-\epsilon/2
\end{align}

By union bound, with probability at least $1-\epsilon$, both these events \eqref{eq:cond1} and \eqref{eq:cond2} hold, in which case:
\begin{align*}
\tilde{X^2} &= (\tilde{a_{uv}} - \frac{\hat{s_{uv}}}{t})^2 \frac{t^2}{\hat{s_{uv}}(t-1)}\\
& = (\hat{a_{uv}} - \nu \cdot N_t - \frac{\hat{s_{uv}}}{t})^2 \frac{t^2}{\hat{s_{uv}}(t-1)}\\
& \le (a_{uv} - \frac{s_{uv}}{t})^2 \frac{t^2}{s_{uv}(t-1)}\\
& = X^2 \le \chi_{1-\epsilon/2}^2(1)
\end{align*}
Finally, we conclude that
\begin{align}
P(\tilde{X^2} > \chi_{1-\epsilon/2}^2(1)) < \epsilon.
\end{align}
\end{proof}
\subsection{Incorporating Relations}

In this section, we describe our \method-R approach, which considers edges in a {\bf relational} manner: that is, it aims to group together edges which are nearby, either temporally or spatially.

\paragraph{Temporal Relations} Rather than just counting edges in the same time tick (as we do in \method), we want to allow for some temporal flexibility: i.e. edges in the recent past should also count toward the current time tick, but modified by a reduced weight. A simple and efficient way to do this using our CMS data structures is as follows: at the end of every time tick, rather than resetting our CMS data structures $a_{uv}$, we reduce all its counts by a fixed fraction $\alpha \in (0, 1)$. This allows past edges to count toward the current time tick, with a diminishing weight.

\paragraph{Spatial Relations} We would like to catch large groups of spatially nearby edges: e.g. a single source IP address suddenly creating a large number of edges to many destinations, or a small group of nodes suddenly creating an abnormally large number of edges between them. A simple intuition we use is that in either of these two cases, we expect to observe {\bf nodes} with a sudden appearance of a large number of edges. Hence, we can use CMS data structures to keep track of edge counts like before, except counting all edges adjacent to any node $u$. Specifically, we create CMS counters $\hat{a_u}$ and $\hat{s_u}$ to approximate the current and total edge counts adjacent to node $u$. Given each incoming edge $(u,v)$, we can then compute three anomalousness scores: one for edge $(u,v)$, as in our previous algorithm; one for node $u$, and one for node $v$. Finally, we combine the three scores by taking their maximum value. Another possibility of aggregating the three scores is to take their sum. Algorithm \ref{alg:midasr} summarizes the resulting \method-R algorithm.

\begin{algorithm}
	\caption{\method-R:\ Incorporating Relations \label{alg:midasr}}
	\KwIn{Stream of graph edges over time}
	\KwOut{Anomaly scores per edge}
	{\bf $\triangleright$ Initialize CMS data structures:} \\
	Initialize CMS for total count $s_{uv}$ and current count $a_{uv}$ \\
	Initialize CMS for total count $s_{u}$ and current count $a_{u}$ \\
	\While{new edge $e=(u,v,t)$ is received:}{
	{\bf $\triangleright$ Update Counts:} \\
	Update CMS data structures for the new edge $uv$\\
	{\bf $\triangleright$ Query Counts:} \\
	Retrieve updated counts $\hat{s_{uv}}$ and $\hat{a_{uv}}$\\
	Retrieve updated counts $\hat{s_u},\hat{s_v},\hat{a_{u}},\hat{a_{v}}$\\
	{\bf $\triangleright$ Compute Edge Scores:}\\
	 $\text{score}(u,v,t) = (\hat{a_{uv}} - \frac{\hat{s_{uv}}}{t})^2 \frac{t^2}{\hat{s_{uv}}(t-1)}$\\
	{\bf $\triangleright$ Compute Node Scores:}\\
	 $\text{score}(u,t) = (\hat{a_{u}} - \frac{\hat{s_{u}}}{t})^2 \frac{t^2}{\hat{s_{u}}(t-1)}$\\
	$\text{score}(v,t) = (\hat{a_{v}} - \frac{\hat{s_{v}}}{t})^2 \frac{t^2}{\hat{s_{v}}(t-1)}$\\
	{\bf $\triangleright$ Final Node Scores:}\\ 
	$\textbf{output} \max\{ \text{score}(u,v,t), \text{score}(u,t), \text{score}(v,t) \}$
	}
\end{algorithm}

\subsection{Time and Memory Complexity}

In terms of memory, both \method\ and \method-R only need to maintain the CMS data structures over time, which are proportional to $O(wb)$, where $w$ and $b$ are the number of hash functions and the number of buckets in the CMS data structures; which is bounded with respect to the data size. 

For time complexity, the only relevant steps in Algorithm \ref{alg:midas} and \ref{alg:midasr} are those that either update or query the CMS data structures, which take $O(w)$ (all other operations run in constant time). Thus, time complexity per update step is $O(w)$. 
\section{Experiments}

In this section, we evaluate the performance of \method\ and \method-R compared to \sedanspot\ on dynamic graphs. We aim to answer the following questions:

\begin{enumerate}[label=\textbf{Q\arabic*.}]
\item {\bf Accuracy:} How accurately does \method\ detect real-world anomalies compared to baselines, as evaluated using the ground truth labels?
\item {\bf Scalability:} How does it scale with input stream length? How does the time needed to process each input compare to baseline approaches?
\item {\bf Real-World Effectiveness:} Does it detect meaningful anomalies in case studies on \emph{Twitter} graphs?
\end{enumerate}

\paragraph{Datasets:}
\emph{DARPA} \cite{lippmann1999results} has $4.5M$ IP-IP communications between $9.4K$ source IP and $23.3K$ destination IP over $87.7K$ minutes. Each communication is a directed edge (srcIP, dstIP, timestamp, attack) where the ground truth attack label indicates whether the communication is an attack or not (anomalies are $60.1\%$ of total).

\emph{TwitterSecurity} \cite{rayana2015less,rayana2016less} has $2.6M$ tweet samples for four months (May-Aug $2014$) containing Department of Homeland Security keywords related to terrorism or domestic security. Entity-entity co-mention temporal graphs are built on daily basis ($80$ time ticks). 

\emph{TwitterWorldCup} \cite{rayana2015less,rayana2016less} has $1.7M$ tweet samples for the World Cup $2014$ season (June $12$-July $13$). The tweets are filtered by popular/official World Cup hashtags, such as \#worldcup, \#fifa, \#brazil, etc.  Similar to TwitterSecurity, entity-entity co-mention temporal graphs are constructed on $5$ minute sample rate ($8640$ time points). 

\paragraph{Baseline:}
As described in our Related Work, only RHSS and \sedanspot\ operate on edge streams and provide a score for each edge. 
\sedanspot\ uses personalised PageRank to detect anomalies in sublinear space and constant time per edge. However, RHSS was evaluated in \cite{eswaran2018sedanspot} on the DARPA dataset and found to have AUC of $0.17$ (lower than chance). Hence, we only compare with \sedanspot.

\paragraph{Evaluation Metrics:}
All the methods output an anomaly score per edge (higher is more anomalous). We calculate the True Positive Rate (TPR) and False Positive Rate (FPR) and plot the ROC curve (TPR vs FPR). We also report the Area under the ROC curve (AUC) and Average Precision Score.

\subsection{Experimental Setup}
All experiments are carried out on a $2.4 GHz$ Intel Core $i9$ processor, $32 GB$ RAM, running OS $X$ $10.15.2$. We implement \method\ and \method-R in C\texttt{++}. We use $2$ hash functions for the CMS data structures, and we set the number of CMS buckets to $2719$ to result in an approximation error of $\nu=0.001$. For \method-R, we set the temporal decay factor $\alpha$ as $0.5$. We used an open-sourced implementation of \sedanspot, provided by the authors, following parameter settings as suggested in the original paper (sample size $500$). 

\subsection{Q1. Accuracy}
Figure \ref{fig:ROC} plots the ROC curve for \method-R, \method\ and \sedanspot\ on the \emph{DARPA} dataset. Figure \ref{fig:AUC}(top) plots accuracy (AUC) vs. running time (log scale, in seconds, excluding I/O). We see that \method\ achieves a much higher accuracy $(=0.91)$ compared to the baseline $(=0.64)$, while also running significantly faster $(0.13s$ vs. $84s)$. This is a $42\%$ accuracy improvement at $644\times$ faster speed. \method-R achieves the highest accuracy $(=0.95)$ which is $48\%$ accuracy improvement compared to the baseline at $215\times$ faster speed.

Figure \ref{fig:AUC}(bottom) plots the average precision score vs. running time. We see that \method\ is more precise $(=0.95)$ compared to the baseline $(=0.75)$. This is a $27\%$ precision improvement. \method-R achieves the highest average precision score $(=0.97)$ which is $29\%$ more precise than \sedanspot.

We see that \method\ and \method-R greatly outperform \sedanspot\ on both accuracy and precision metrics.
\begin{figure}[!htb]
        \center{\includegraphics[width=\columnwidth]
        {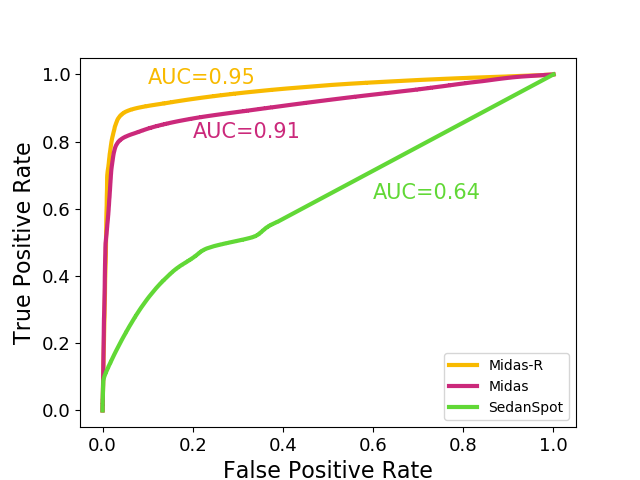}}
        \caption{\label{fig:ROC} ROC for \emph{DARPA} dataset}
\end{figure}

\begin{figure}[!htb]
        \center{\includegraphics[width=\columnwidth]
        {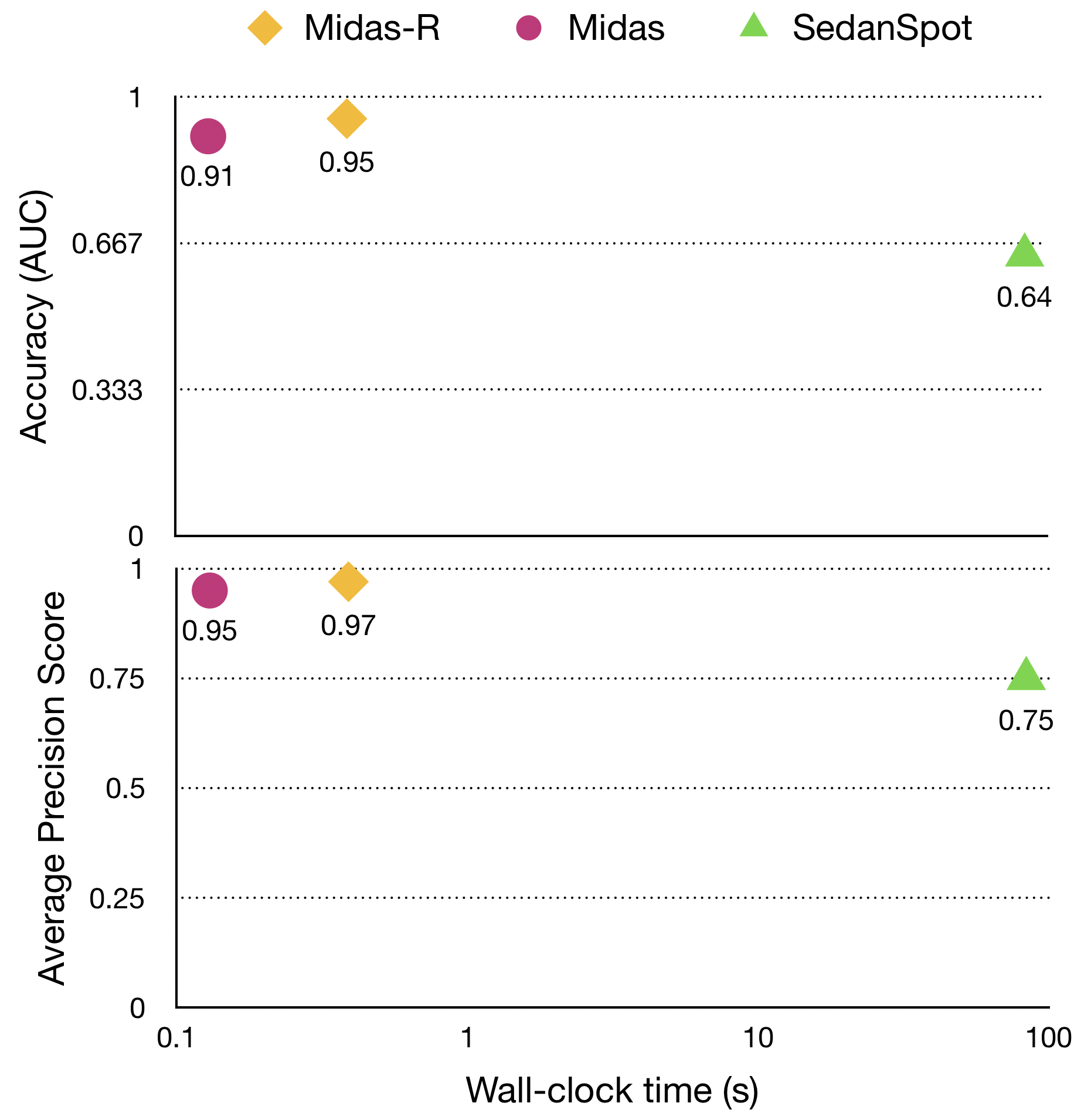}}
        \caption{\label{fig:AUC} (top) Accuracy (AUC) vs time, (bottom) Average Precision Score vs time}
\end{figure}

\subsection{Q2. Scalability}
Figure \ref{fig:scaling} shows the scalability of \method\ and \method-R. We plot the wall-clock time needed to run on the (chronologically) first $2^{12}, 2^{13},2^{14},...,2^{22}$ edges of the \emph{DARPA} dataset. This confirms the linear scalability of \method\ and \method-R with respect to the number of edges in the input dynamic graph due to its constant processing time per edge. Note that both \method\ and \method-R process $4M$ edges within $0.5$ second, allowing real-time anomaly detection. 

Figure \ref{fig:frequency} plots the number of edges (in millions) and time to process each edge for \emph{DARPA} dataset. \method\ processes $4.4M$ edges within $1\mu$s each and $0.15M$ edges within $2\mu$s each. \method-R processes $4.3M$ edges within $1\mu$s each and $0.23M$ edges within $2\mu$s each.

Table \ref{tab:times} shows the time it takes \sedanspot, \method\ and \method-R to run on the \emph{TwitterWorldCup}, \emph{TwitterSecurity} and \emph{DARPA} datasets. For \emph{TwitterWorldCup} dataset, we see that \method-R is $162\times$ faster than \sedanspot\ $(0.17s$ vs. $27.58s)$ and \method\ is $460\times$ faster than \sedanspot $(0.06s$ vs $27.58s)$. For \emph{TwitterSecurity} dataset, we see that \method-R is $177\times$ faster than \sedanspot\ $(0.23s$ vs. $40.71s)$ and \method\ is $509\times$ faster than \sedanspot $(0.08s$ vs $40.71s)$. For the \emph{DARPA} dataset, we see that \method-R is $215\times$ faster than \sedanspot\ $(0.39s$ vs. $83.66s)$ and \method\ is $644\times$ faster than \sedanspot $(0.13s$ vs $83.66s)$.

\sedanspot\ requires several subprocesses (hashing, random-walking, reordering, sampling, etc), resulting in the large computation time. \method\ and \method-R are both scalable and fast.

\begin{figure}[!htb]
        \center{\includegraphics[width=\columnwidth]
        {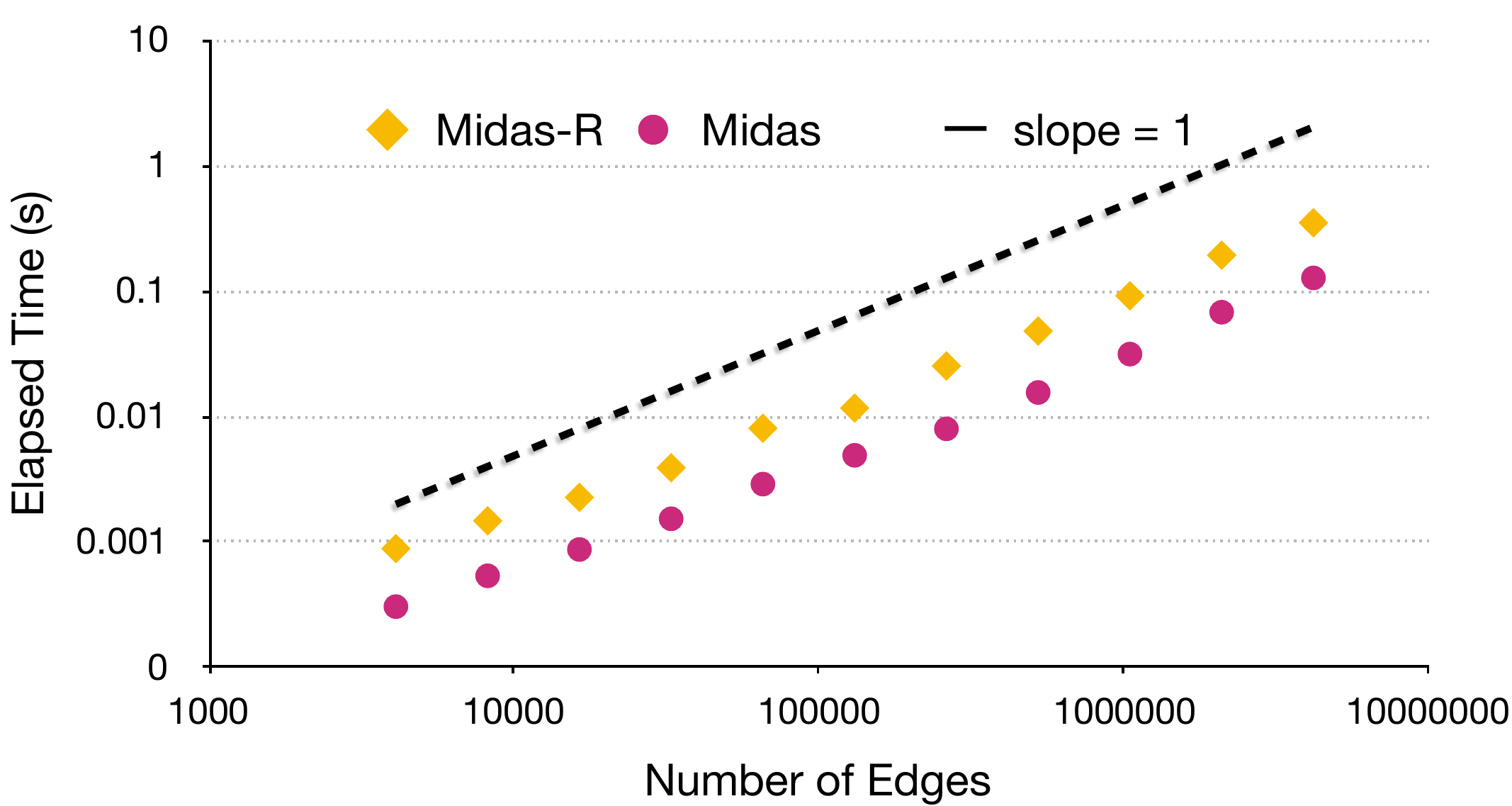}}
        \caption{\label{fig:scaling} \method\ and \method-R scale linearly with the number of edges in the input dynamic graph.}
\end{figure}

\begin{figure}[!htb]
        \center{\includegraphics[width=\columnwidth]
        {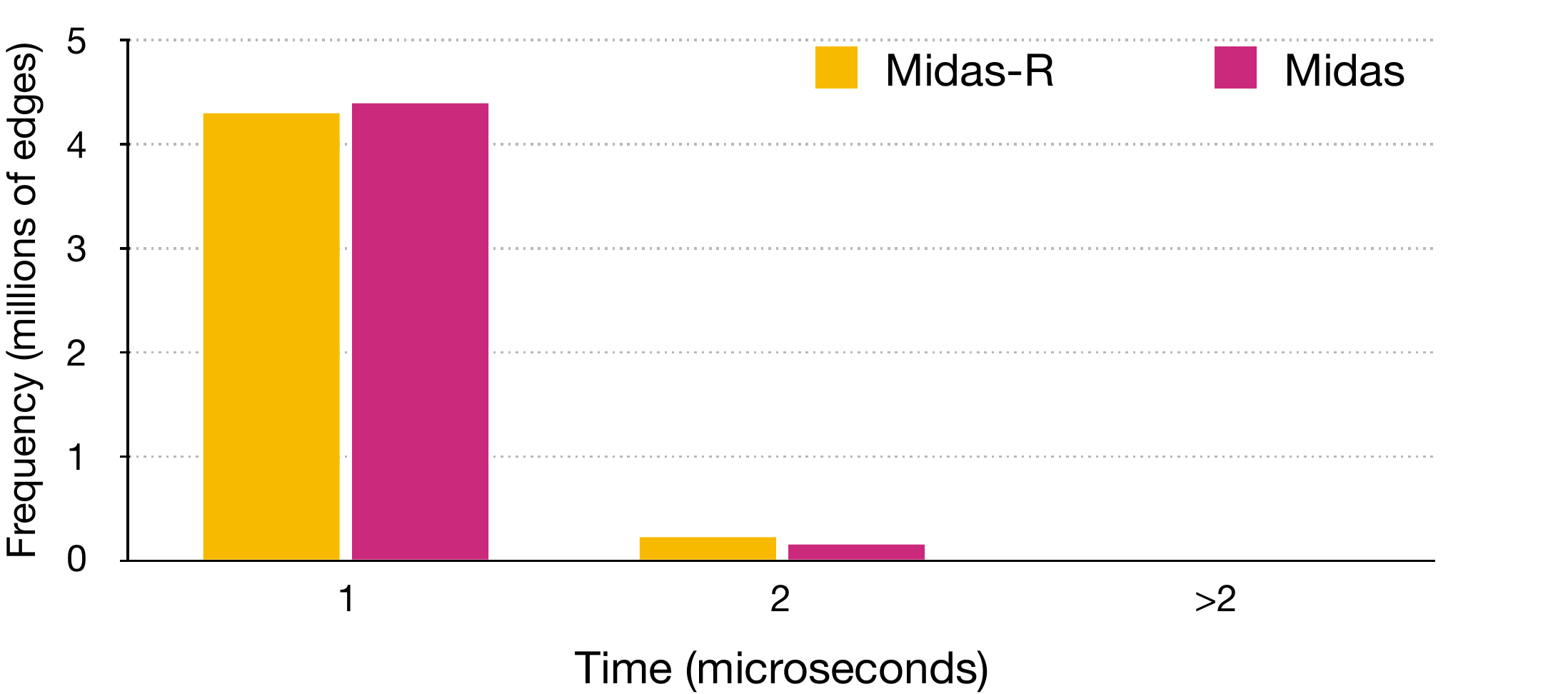}}
        \caption{\label{fig:frequency} Distribution of processing times for $\sim4.5M$ edges of \emph{DARPA} dataset.}
\end{figure}

\begin{table}[!ht]
\centering
\caption{Running time for different datasets in seconds}
\label{tab:times}
\begin{tabular}{@{}rccc@{}}
\toprule
 & \sedanspot
 & \method
 & \method-R \\ \midrule
\textbf{TwitterWorldCup} & $27.58$s & $0.06$s  & $0.17$s \\
\textbf{TwitterSecurity} & $40.71$s & $0.08$s & $0.23$s \\
\textbf{DARPA} & $83.66$s & $0.13$s & $0.39$s \\
\bottomrule
\end{tabular}
\end{table}

\subsection{Q3. Real-World Effectiveness}
We measure anomaly scores using \method, \method-R and \sedanspot\ on the \emph{TwitterSecurity} dataset. Figure \ref{fig:security} plots anomaly scores vs. day (during the four months of $2014$). To visualise, we aggregate edges occurring in each day by taking the max anomalousness score per day, for a total of $90$ days. Anomalies correspond to major world news such as Mpeketoni attack (Event $6$) or Soma Mine explosion (Event $1$). \method\ and \method-R show similar trends whereas \sedanspot\ misses some anomalous events (Events $2,7$), and outputs many high scores unrelated to any true events. This is also reflected in the low accuracy and precision of \sedanspot\ in Figure \ref{fig:AUC}. The anomalies detected by \method\ and \method-R coincide with major events in the \emph{TwitterSecurity} timeline as follows:

\begin{footnotesize}
\begin{enumerate}
\item 13-05-2014. Turkey Mine Accident, Hundreds Dead
\item 24-05-2014. Raid.
\item 30-05-2014. Attack/Ambush.\\
03-06-14. Suicide bombing
\item 09-06-14. Suicide/Truck bombings.
\item 10-06-2014. Iraqi Militants Seized Large Regions.\\
11-06-2014. Kidnapping
\item 15-06-14. Attack
\item 26-06-14. Suicide Bombing/Shootout/Raid
\item 03-07-14. Israel Conflicts with Hamas in Gaza.
\item 18-07-14. Airplane with 298 Onboard was Shot Down over Ukraine.
\item 30-07-14. Ebola Virus Outbreak.

\end{enumerate}
\end{footnotesize}
This shows the effectiveness of \method\ and \method-R for catching real-world anomalies.

\textbf{Microcluster anomalies:} Figure \ref{fig:micro} corresponds to Event $7$ in the \emph{TwitterSecurity} dataset. All single edges are equivalent to 444 edges and double edges are equivalent to 888 edges between the nodes. This suddenly arriving (within 1 day) group of suspiciously similar edges is an example of a microcluster anomaly which \method-R detects, but \sedanspot\ misses.

\begin{figure}[!htb]
        \center{\includegraphics[width=\columnwidth]
        {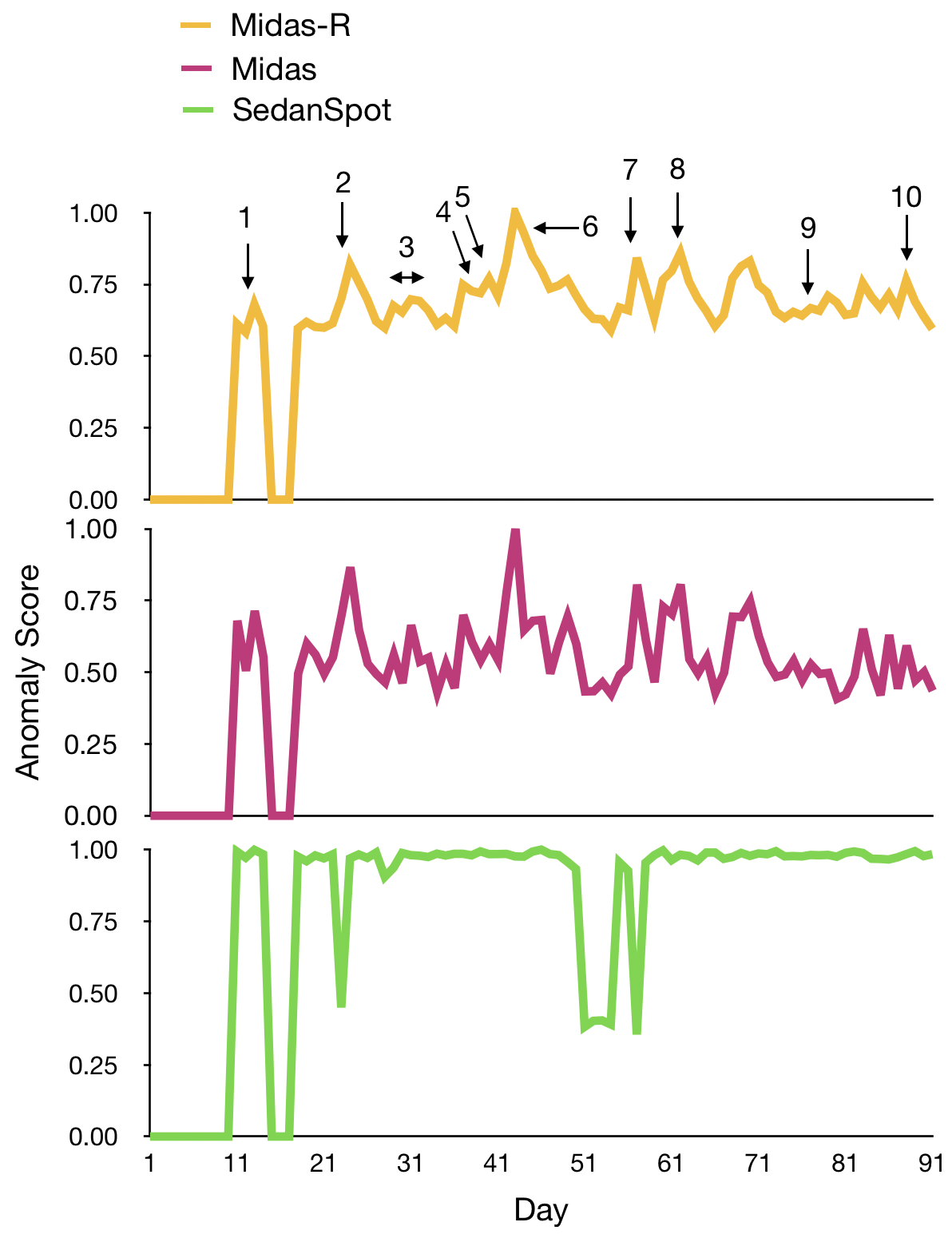}}
        \caption{\label{fig:security} Anomalies detected by \method\ and \method-R correspond to major security-related events in \emph{TwitterSecurity}.}
\end{figure}

\begin{figure}[!htb]
        \center{\includegraphics[width=\columnwidth]
        {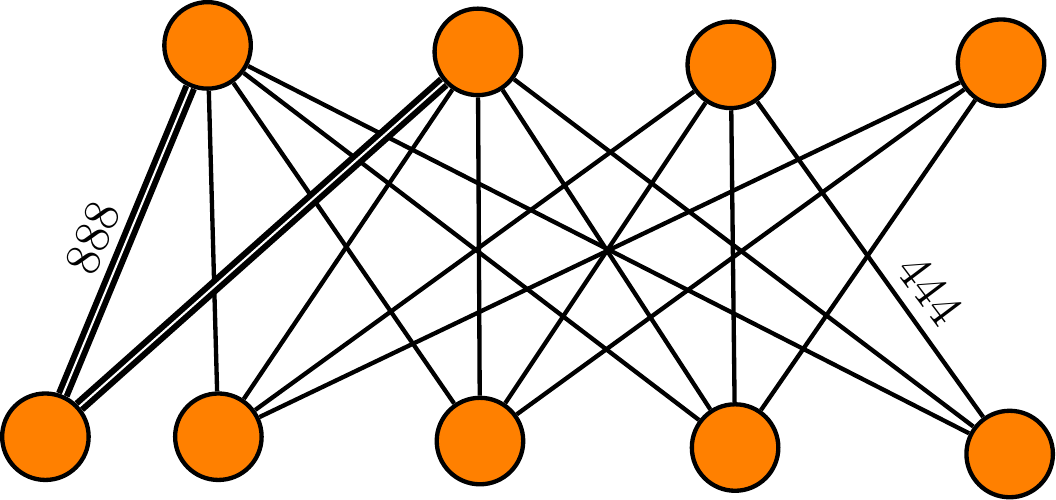}}
        \caption{\label{fig:micro} Microcluster Anomaly in \emph{TwitterSecurity}}
\end{figure}

\section{Conclusion}
In this paper, we proposed \method\ and \method-R for microcluster based detection of anomalies in edge streams. Future work could consider more general types of data, including heterogeneous graphs or tensors. 
Our contributions are as follows:
\begin{enumerate}
\item Streaming Microcluster Detection: We propose a novel streaming approach for detecting microcluster anomalies, requiring constant time and memory. 
\item Theoretical Guarantees: In Theorem \ref{thm:bound}, we show guarantees on the false positive probability of \method.
\item Effectiveness: Our experimental results show that \method\ outperforms baseline approaches by $42\%$-$48\%$ accuracy (in terms of AUC), and processes the data $162-644$ times faster than baseline approaches.
\end{enumerate}

\section{Acknowledgments}
This work was supported in part by NUS ODPRT Grant R-252-000-A81-133.

\bibliographystyle{aaai}
\bibliography{paper}

\begin{thebibliography}{}

\bibitem[\protect\citeauthoryear{Akoglu, McGlohon, and
  Faloutsos}{2010}]{akoglu2010oddball}
Akoglu, L.; McGlohon, M.; and Faloutsos, C.
\newblock 2010.
\newblock Oddball: Spotting anomalies in weighted graphs.
\newblock In {\em PAKDD}.

\bibitem[\protect\citeauthoryear{Akoglu, Tong, and
  Koutra}{2015}]{akoglu2015graph}
Akoglu, L.; Tong, H.; and Koutra, D.
\newblock 2015.
\newblock Graph based anomaly detection and description: a survey.
\newblock {\em Data Mining and Knowledge Discovery} 29(3):626--688.

\bibitem[\protect\citeauthoryear{Beutel \bgroup et al\mbox.\egroup
  }{2013}]{beutel2013copycatch}
Beutel, A.; Xu, W.; Guruswami, V.; Palow, C.; and Faloutsos, C.
\newblock 2013.
\newblock Copycatch: stopping group attacks by spotting lockstep behavior in
  social networks.
\newblock In {\em WWW}.

\bibitem[\protect\citeauthoryear{Chakrabarti}{2004}]{chakrabarti2004autopart}
Chakrabarti, D.
\newblock 2004.
\newblock Autopart: Parameter-free graph partitioning and outlier detection.
\newblock In {\em PKDD}.

\bibitem[\protect\citeauthoryear{Cormode and
  Muthukrishnan}{2005}]{cormode2005improved}
Cormode, G., and Muthukrishnan, S.
\newblock 2005.
\newblock An improved data stream summary: the count-min sketch and its
  applications.
\newblock {\em Journal of Algorithms} 55(1):58--75.

\bibitem[\protect\citeauthoryear{Eswaran and
  Faloutsos}{2018}]{eswaran2018sedanspot}
Eswaran, D., and Faloutsos, C.
\newblock 2018.
\newblock Sedanspot: Detecting anomalies in edge streams.
\newblock In {\em 2018 IEEE International Conference on Data Mining (ICDM)},
  953--958.
\newblock IEEE.

\bibitem[\protect\citeauthoryear{Eswaran \bgroup et al\mbox.\egroup
  }{2018}]{eswaran2018spotlight}
Eswaran, D.; Faloutsos, C.; Guha, S.; and Mishra, N.
\newblock 2018.
\newblock Spotlight: Detecting anomalies in streaming graphs.
\newblock In {\em KDD}.

\bibitem[\protect\citeauthoryear{Gupta \bgroup et al\mbox.\egroup
  }{2012}]{Gupta}
Gupta, M.; Gao, J.; Sun, Y.; and Han, J.
\newblock 2012.
\newblock Integrating community matching and outlier detection for mining
  evolutionary community outliers.
\newblock In {\em Proceedings of the 18th ACM SIGKDD International Conference
  on Knowledge Discovery and Data Mining}, KDD '12,  859--867.
\newblock New York, NY, USA: ACM.

\bibitem[\protect\citeauthoryear{Hooi \bgroup et al\mbox.\egroup
  }{2017}]{hooi2017graph}
Hooi, B.; Shin, K.; Song, H.~A.; Beutel, A.; Shah, N.; and Faloutsos, C.
\newblock 2017.
\newblock Graph-based fraud detection in the face of camouflage.
\newblock {\em TKDD} 11(4):44.

\bibitem[\protect\citeauthoryear{Jiang \bgroup et al\mbox.\egroup
  }{2016}]{jiang2016catching}
Jiang, M.; Cui, P.; Beutel, A.; Faloutsos, C.; and Yang, S.
\newblock 2016.
\newblock Catching synchronized behaviors in large networks: A graph mining
  approach.
\newblock {\em TKDD} 10(4):35.

\bibitem[\protect\citeauthoryear{Kleinberg}{1999}]{kleinberg1999authoritative}
Kleinberg, J.~M.
\newblock 1999.
\newblock Authoritative sources in a hyperlinked environment.
\newblock {\em JACM} 46(5):604--632.

\bibitem[\protect\citeauthoryear{Koutra, Vogelstein, and
  Faloutsos}{2013}]{koutra2013deltacon}
Koutra, D.; Vogelstein, J.~T.; and Faloutsos, C.
\newblock 2013.
\newblock Deltacon: A principled massive-graph similarity function.
\newblock {\em arXiv preprint arXiv:1304.4657}.

\bibitem[\protect\citeauthoryear{Lippmann \bgroup et al\mbox.\egroup
  }{1999}]{lippmann1999results}
Lippmann, R.; Cunningham, R.~K.; Fried, D.~J.; Graf, I.; Kendall, K.~R.;
  Webster, S.~E.; and Zissman, M.~A.
\newblock 1999.
\newblock Results of the darpa 1998 offline intrusion detection evaluation.
\newblock In {\em Recent advances in intrusion detection}, volume~99,
  829--835.

\bibitem[\protect\citeauthoryear{Ranshous \bgroup et al\mbox.\egroup
  }{2016}]{ranshous2016scalable}
Ranshous, S.; Harenberg, S.; Sharma, K.; and Samatova, N.~F.
\newblock 2016.
\newblock A scalable approach for outlier detection in edge streams using
  sketch-based approximations.
\newblock In {\em Proceedings of the 2016 SIAM International Conference on Data
  Mining},  189--197.
\newblock SIAM.

\bibitem[\protect\citeauthoryear{Rayana and Akoglu}{2015}]{rayana2015less}
Rayana, S., and Akoglu, L.
\newblock 2015.
\newblock Less is more: Building selective anomaly ensembles with application
  to event detection in temporal graphs.
\newblock In {\em Proceedings of the 2015 SIAM International Conference on Data
  Mining},  622--630.
\newblock SIAM.

\bibitem[\protect\citeauthoryear{Rayana and Akoglu}{2016}]{rayana2016less}
Rayana, S., and Akoglu, L.
\newblock 2016.
\newblock Less is more: Building selective anomaly ensembles.
\newblock {\em ACM Transactions on Knowledge Discovery from Data (TKDD)}
  10(4):42.

\bibitem[\protect\citeauthoryear{Shin \bgroup et al\mbox.\egroup
  }{2017}]{shin2017densealert}
Shin, K.; Hooi, B.; Kim, J.; and Faloutsos, C.
\newblock 2017.
\newblock Densealert: Incremental dense-subtensor detection in tensor streams.
\newblock {\em KDD}.

\bibitem[\protect\citeauthoryear{Shin, Eliassi-Rad, and
  Faloutsos}{2018}]{shin2018patterns}
Shin, K.; Eliassi-Rad, T.; and Faloutsos, C.
\newblock 2018.
\newblock Patterns and anomalies in k-cores of real-world graphs with
  applications.
\newblock {\em KAIS} 54(3):677--710.

\bibitem[\protect\citeauthoryear{Sricharan and Das}{2014}]{Sricharan}
Sricharan, K., and Das, K.
\newblock 2014.
\newblock Localizing anomalous changes in time-evolving graphs.
\newblock In {\em Proceedings of the 2014 ACM SIGMOD International Conference
  on Management of Data}, SIGMOD '14,  1347--1358.
\newblock New York, NY, USA: ACM.

\bibitem[\protect\citeauthoryear{Sun \bgroup et al\mbox.\egroup
  }{2007}]{sun2007graphscope}
Sun, J.; Faloutsos, C.; Papadimitriou, S.; and Yu, P.
\newblock 2007.
\newblock Graphscope: parameter-free mining of large time-evolving graphs.
\newblock In {\em Proceedings of the 13th ACM SIGKDD international conference
  on Knowledge discovery and data mining},  687--696.

\bibitem[\protect\citeauthoryear{Sun, Tao, and Faloutsos}{2006}]{sun2006beyond}
Sun, J.; Tao, D.; and Faloutsos, C.
\newblock 2006.
\newblock Beyond streams and graphs: dynamic tensor analysis.
\newblock In {\em KDD}.

\bibitem[\protect\citeauthoryear{Tong and Lin}{2011}]{tong2011non}
Tong, H., and Lin, C.-Y.
\newblock 2011.
\newblock Non-negative residual matrix factorization with application to graph
  anomaly detection.
\newblock In {\em SDM}.

\bibitem[\protect\citeauthoryear{Yoon \bgroup et al\mbox.\egroup
  }{2019}]{yoon2019fast}
Yoon, M.; Hooi, B.; Shin, K.; and Faloutsos, C.
\newblock 2019.
\newblock Fast and accurate anomaly detection in dynamic graphs with a
  two-pronged approach.
\newblock In {\em Proceedings of the 25th ACM SIGKDD International Conference
  on Knowledge Discovery \& Data Mining},  647--657.
\newblock ACM.

\bibitem[\protect\citeauthoryear{Yu \bgroup et al\mbox.\egroup
  }{2013}]{yu2013anomalous}
Yu, W.; Aggarwal, C.~C.; Ma, S.; and Wang, H.
\newblock 2013.
\newblock On anomalous hotspot discovery in graph streams.
\newblock In {\em ICDM}.

\end{thebibliography}

\end{document}